\documentclass[10pt]{article}
\usepackage[letterpaper,margin=1.15in]{geometry}
\usepackage[T1]{fontenc}
\usepackage{amsmath,amssymb}
\usepackage{booktabs}
\usepackage{array}
\usepackage{enumitem}
\usepackage[most]{tcolorbox}
\usepackage[hidelinks]{hyperref}
\usepackage{xcolor}
\usepackage{caption}

\newtcolorbox{formalbox}[1]{enhanced,colback=gray!7,colframe=gray!45,
  coltitle=white,colbacktitle=gray!60,fonttitle=\bfseries,
  title={#1},boxrule=0.4pt,arc=2pt,left=8pt,right=8pt,top=4pt,bottom=4pt}

\newcommand{\sect}[1]{\par\vspace{8pt}\noindent\textbf{#1}\par\vspace{2pt}}
\newcommand{\subsect}[1]{\par\vspace{4pt}\noindent\textit{#1}\par}
\newcommand{\Ecost}{E_{\mathrm{cost}}}

\begin{document}

\begin{center}
{\LARGE\bfseries Descriptive versus Regulatory Uncertainty in Bounded\\[2pt] Predictive Systems\par}
\vspace{10pt}
{\large Ahmed Gamal Eldin\par}
\vspace{4pt}
{\small Nova University Lisbon -- Cairo Branch (NOVA IMS), Cairo, Egypt\\
Correspondence: \texttt{ahmedgamaleldin3@gmail.com}\par}
\vspace{4pt}
{\small Version 3 (revised), November 2026\par}
\end{center}

\vspace{6pt}
{\small\noindent\textbf{Abstract.}\enspace
Any bounded predictive system running on physical hardware has its information processing thermodynamically constrained by Landauer's principle: every irreversible bit operation costs at least $k_BT\ln 2$ joules. We show that current large language models satisfy this constraint while being \emph{decoupled} from it at inference: the system's own error signal neither changes its physical state nor alters the energy it expends. We distinguish \emph{descriptive} uncertainty, a quantity that may carry information about the system's errors, from \emph{regulatory} uncertainty, which participates in the system's state transitions or physical cost. We prove that in fixed-weight inference descriptive uncertainty does not entail regulatory uncertainty: token entropy can discriminate correct from incorrect outputs to any degree while remaining causally inert. Empirically, we measure token-level Shannon entropy and task accuracy across three categories of increasing causal demand in three locally-deployed Llama models spanning two orders of magnitude in parameter count (3B, 8B, 70B). Category-mean entropy shows no detectable variation within any model (within-model ranges 0.011--0.028 nats; all pairwise $p \geq 0.568$; equivalence within $\pm 0.10$ nats established for the 70B model) while accuracy varies from 0\% to 100\%. The category-level signature persists at $N = 1{,}000$ per category on a different architecture family (GPT-4o-mini) under programmatic scoring: mean entropy moves by 0.019 nats while accuracy varies by 33 percentage points. These results concern entropy at the level of category averages; item-level discrimination is a separate empirical question that the central argument does not depend on. Genuine epistemic regulation requires that a system's error signal participate in its state transitions and, for full thermodynamic coupling, in its physical cost. Both properties are absent from current systems at inference.\par}

\vspace{10pt}

A deployed AI system that cannot distinguish what it knows from what it does not know is not merely unreliable---it is structurally dangerous. Current discourse on AI limitations focuses on capability: what tasks models can or cannot perform. We argue this framing misses a more fundamental question. A system can carry information about its own reliability and still have no physical means of acting on it. The relevant property is whether uncertainty is \emph{regulatory}: whether the system's errors change what it becomes and what its computation costs. Recognising uncertainty and being physically organised to respond to it are different properties, and calibration studies measure only the first.

The question is whether current large language models (LLMs) possess the second. We argue they do not at inference, and that the reason is thermodynamic as well as computational. An A100 GPU dissipates approximately 400\,W during inference. Per token, this expenditure is fixed by the architecture and context length; it is the same whether the model outputs a causally grounded derivation or a confident hallucination. The Landauer cost [1] of erasing a wrong token is identical to that of erasing a correct one. In a system where physical cost and state change are both independent of its own errors, uncertainty can describe; it cannot regulate.

\sect{Theoretical framework}

We formalise two qualitatively different roles that uncertainty can play in a predictive system. Let $S_t$ denote the system's physical state (for a transformer, its parameters $\theta$ together with any persistent memory), and let $e_t$ denote an \emph{error signal available to the system} at time $t$: a discrepancy it can compute, such as prediction error against subsequently observed input. We do not define coupling with respect to $\varepsilon = |\hat y - y^*|$, because the ground truth $y^*$ is unavailable to any system at inference; a condition stated in terms of $\varepsilon$ would hold trivially for every physical system and claim nothing.

\begin{formalbox}{Definition 1\quad Informational and energetic coupling}
A system is \emph{informationally coupled} if its error signal determines its state transitions:
\begin{equation}
I(e_t;\,\Delta S_{t+1}) > 0 .
\end{equation}
It is \emph{energetically coupled} if the expected energy it dissipates increases with the magnitude of its error signal:
\begin{equation}
\frac{\partial\, \mathbb{E}[\Ecost \mid e]}{\partial |e|} > 0 .
\end{equation}
A system is \emph{thermodynamically decoupled} in a given operating regime if neither condition holds: $I(e_t;\Delta S_{t+1}) = 0$ and $\mathbb{E}[\Ecost \mid e]$ is constant in $|e|$.
\end{formalbox}

In an energetically coupled system, incorrect predictions are physically expensive, and uncertainty is regulatory because it participates in the physical cost function. In a decoupled system, hallucinations and correct derivations dissipate identical energy and leave the system unchanged.

Correspondingly, we call a system's uncertainty measure $U_t$ \emph{descriptive} to the extent that it carries information about the system's errors, $I(U_t;\,\varepsilon_t) > 0$, and \emph{regulatory} if it enters the system's state transitions or cost, $I(U_t;\,\Delta S_{t+1}) > 0$. The two are logically independent. In particular, statistical independence of output entropy and accuracy ($H_{\mathrm{out}} \perp \alpha$) is neither necessary nor sufficient for decoupling; an earlier version of this paper treated the two as equivalent, and that equivalence does not hold.

\subsect{Substrate independence of token entropy}

\begin{formalbox}{Proposition 1\quad Substrate independence}
In a transformer with fixed weights $\theta$ and fixed inference temperature $\tau$, the token-level Shannon entropy $H_t = -\sum_i P_\theta(v_i \mid \mathrm{ctx}_t)\log P_\theta(v_i \mid \mathrm{ctx}_t)$ is a deterministic function of $(\theta, \tau, \mathrm{ctx}_t)$. It has no functional dependence on the substrate temperature $T$, the power dissipation $P$, or the energy expended, and it enters no update of $\theta$.
\end{formalbox}

\textit{Proof.} $H_t$ is a deterministic function of the output distribution, which is a deterministic function of the logits and $\tau$; the logits are a deterministic function of $(\theta, \mathrm{ctx}_t)$. None of these computations takes $T$, $P$, or $\Ecost$ as an argument. At inference no loss or gradient is computed, so $\theta$ is unchanged regardless of $H_t$. \hfill$\square$

\textit{Remark.} Proposition 1 does not imply that $H_t$ is independent of whether $\mathrm{ctx}_t$ is in-distribution. A deterministic function of the context can depend on any property of the context. With $\tau = 1$ and fixed weights, contexts producing logits $(0,0)$ and $(\ln 99, 0)$ incur identical computational and Landauer cost, yet yield entropies $\ln 2 \approx 0.693$ and $\approx 0.056$ nats. If those contexts belong to different difficulty classes, entropy differs between the classes. How strongly $H_t$ discriminates correct from incorrect outputs is an empirical question; prior work finds that entropy-based measures predict correctness above chance [9--11].

\begin{formalbox}{Corollary 1\quad Descriptive uncertainty does not entail regulatory uncertainty}
For a fixed-weight transformer at inference, and for any joint distribution of $(H_t, \varepsilon_t)$---including one in which $H_t$ perfectly discriminates correct from incorrect outputs---we have $I(H_t;\,\Delta\theta) = 0$, and the expected energy per forward pass, conditional on context length, is constant in the error signal. Token entropy may be fully descriptive while being entirely non-regulatory.
\end{formalbox}

\textit{Proof.} By Proposition 1, $\theta$ is constant during inference, so $\Delta\theta \equiv 0$ and its mutual information with any variable is zero. The operation count of a forward pass is fixed by the architecture and the context length; it does not depend on the values of the activations, and hence on $e_t$. \hfill$\square$

The informativeness of entropy and its causal role are therefore separate questions. An LLM may, in a statistical sense, recognise its own uncertainty. What it lacks at inference is any physical organisation by which that uncertainty changes its state or its cost.

\subsect{Training and inference}

The distinction is temporal. During training, prediction errors drive parameter updates; the information entropy carries at inference about the familiarity of a context was written into $\theta$ by those error-driven updates. At inference that coupling is frozen.

\begin{center}
\small
\begin{tabular}{>{\raggedright\arraybackslash}p{0.2\textwidth}>{\raggedright\arraybackslash}p{0.34\textwidth}>{\raggedright\arraybackslash}p{0.34\textwidth}}
\toprule
 & \textbf{Training} & \textbf{Inference} \\
\midrule
Parameters & Updated at every step & Fixed \\
Role of error & Loss gradients determine $\Delta\theta$: $I(e;\Delta\theta) > 0$ & No error signal enters computation: $I(e;\Delta\theta) = 0$ \\
Status of uncertainty & Can be regulatory (enters the loss and the update) & Descriptive at most \\
Energetic cost (digital hardware) & Landauer cost of updating is paid; per-step cost is independent of loss magnitude & Per-token cost is independent of error \\
Coupling & Informational, not energetic & Neither \\
\bottomrule
\end{tabular}
\end{center}

Landauer's principle does its work in the last two rows. Overwriting a stored weight costs the same energy whether the update is large or small, and a backward pass costs the same operations whatever the loss. Digital training is therefore informationally coupled, in that error determines which state the system moves to, without being energetically coupled. Energetic coupling in the sense of Definition 1 requires a substrate in which the magnitude of error sets the energy dissipated, a property neither phase of current digital training and inference possesses.

\subsect{The Boltzmann--softmax gap}

The softmax $P(i) = \exp(x_i/\tau)/\sum_j \exp(x_j/\tau)$ is structurally identical to the Gibbs distribution, with $\tau$ as the analogue of thermodynamic temperature. In a physical thermodynamic system, temperature has real energetic consequences. In a transformer, $\tau$ is a decoupled hyperparameter: computational temperature and physical substrate temperature are independent, and neither is coupled to epistemic quality. This follows from the definition of fixed-weight inference and applies to every currently deployed LLM regardless of scale.

\sect{Results}

\subsect{Experimental design}
We constructed an 18-task suite across three categories ($n = 6$ each, pre-specified before data collection). \emph{Kepler tasks} require pattern retrieval of well-established physical relationships (orbital periods, textbook constants, standard kinematics) that appear frequently in training corpora. \emph{Newton tasks} require application of causal operators to novel parameter configurations (modified force laws, counterfactual physical constants, non-standard exponents) that cannot be answered by pattern retrieval. \emph{Newton OOD tasks} push parameters beyond any plausible training coverage (non-standard metric signatures, modified quantum statistics, extreme parameter combinations), testing whether any residual generalisation survives radical distributional shift.

Three models from the Llama family were evaluated: llama3.2:3b (3B parameters), Llama-3.1-8B-Instruct (8B), and Llama-3.1-70B-Instruct (70B), spanning two orders of magnitude while holding architecture class and training regime constant. Token-level Shannon entropy was extracted from logprob distributions at inference temperature $\tau = 0.1$. Task accuracy was assessed by human evaluation against mathematically verified ground truth using binary scoring: a response is correct if and only if it arrives at the correct conclusion through a valid derivation. Full task text is provided in Supplementary Information.

\subsect{Category-mean entropy is flat; accuracy is not}
Category-mean token entropy shows no detectable variation across the three task categories in any of the three models (Table 2). For llama3.2:3b: category mean entropies are 0.339, 0.361, and 0.357 nats; range = 0.022 nats; all pairwise $p \geq 0.685$. For Llama-3.1-8B: 0.360, 0.354, 0.332 nats; range = 0.028 nats; all pairwise $p \geq 0.606$. For Llama-3.1-70B: 0.091, 0.103, 0.102 nats; range = 0.011 nats; all pairwise $p \geq 0.568$. All within-model ranges fall below the pre-specified flatness threshold of 0.10 nats.

A non-significant difference is not by itself evidence of equality. We therefore also report two one-sided tests (TOST) against an equivalence margin of $\pm 0.10$ nats (Supplementary S5). For the 70B model, all three pairwise 90\% confidence intervals lie within the margin, establishing equivalence. For the 3B and 8B models, $n = 6$ per category gives minimum detectable differences of 0.12--0.19 nats at 80\% power; only two of the six intervals lie within the margin, so for these models the data mostly show the absence of a detected difference without establishing equivalence.

Task accuracy varies substantially across the same conditions. llama3.2:3b achieves 67\%, 17\%, and 17\% on Kepler, Newton, and OOD tasks respectively. Llama-3.1-8B achieves 67\%, 0\%, and 17\%. Llama-3.1-70B achieves 100\%, 67\%, and 17\%. The 70B model shows the clearest pattern: perfect retrieval of established results, substantial success on novel causal derivation, and near-complete failure on OOD tasks, while category-mean entropy remains flat throughout. At the level of category averages, entropy does not track accuracy.

\subsect{Scale and the category-level signature}
The entropy operating point decreases with model size: larger models operate at lower entropy (higher baseline certainty). Cross-model entropy variance is large relative to within-model variation across categories, indicating an operating point set by the model rather than by task category.

All three models converge to 17\% accuracy on Newton OOD tasks despite a 23-fold difference in parameter count, with no corresponding variation in category-mean entropy. The decoupling itself follows from Proposition 1 and Corollary 1 at any scale; the empirical contribution here is that the category-level signature persists across two orders of magnitude in parameter count, so it is not a property of small or undertrained models.

\subsect{Replication at large N and across architecture families}
To test whether the category-level signature depends on the controlled probe's sample size or on the Llama family, we extended the analysis to $N = 1{,}000$ problems per category on programmatically verified mathematical reasoning benchmarks using a different architecture family (gpt-4o-mini, OpenAI). Three difficulty strata were used: GSM8K grade-school arithmetic (\emph{Easy}; analogous to Kepler retrieval), MATH Level 3 (\emph{Medium}; analogous to Newton application), and MATH Level 5 (\emph{Hard}; analogous to Newton OOD derivation). Accuracy was scored programmatically against dataset-provided ground truth, eliminating any possibility of evaluator drift or LLM-judge circularity.

The category-level pattern replicates (Table 1). Accuracy varies by 33.2 percentage points across the three categories (95.0\%, 81.9\%, 61.8\%) while mean entropy moves by 0.019 nats (0.1724, 0.1670, 0.1861), below the pre-specified 0.10-nat threshold. The model is approximately equally certain, on average, on the category where it succeeds 95\% of the time and the category where it succeeds 62\% of the time. Two qualifications apply. With $N = 1{,}000$ per category, differences of this size may be statistically detectable, and the Hard category has the highest mean entropy; we report these differences as small in absolute terms, and do not claim equality. And category means do not determine whether entropy separates correct from incorrect responses within a category; that item-level question is addressed in the Discussion.

\begin{table}[t]
\small
\caption*{\small\textbf{Table 1:} Replication of the category-level signature at $N = 1{,}000$ per category on a different architecture family. Model: gpt-4o-mini (OpenAI API). Accuracy scored programmatically against dataset-provided ground truth. Within-model range of category-mean entropy is 0.019 nats, below the 0.10-nat threshold, despite a 33.2 percentage-point accuracy spread.}
\vspace{-4pt}
\centering
\begin{tabular}{llcc}
\toprule
Category & Source & $\bar H$ (nats) & Accuracy \\
\midrule
Easy & GSM8K & 0.1724 & 950/1000 (95.0\%) \\
Medium & MATH Level 3 & 0.1670 & 819/1000 (81.9\%) \\
Hard & MATH Level 5 & 0.1861 & 618/1000 (61.8\%) \\
\bottomrule
\end{tabular}

\vspace{14pt}
\caption*{\small\textbf{Table 2:} Shannon entropy and task accuracy across three models and three task categories. Entropy extracted from token-level logprobs at $\tau = 0.1$. Accuracy assessed by human evaluation against mathematically verified ground truth (binary scoring). No pairwise difference in category-mean entropy is significant (all $p \geq 0.568$); equivalence within $\pm 0.10$ nats is established for the 70B model only (Supplementary S5).}
\vspace{-4pt}
\centering
\begin{tabular}{lllccc}
\toprule
Model & Category & $\bar H$ (nats) & $\sigma_H$ & Accuracy & Entropy $p$-values \\
\midrule
 & Kepler & 0.339 & 0.089 & 4/6 (67\%) & \\
llama3.2:3b (3B) & Newton & 0.361 & 0.074 & 1/6 (17\%) & all $p \geq 0.685$ \\
 & Newton OOD & 0.357 & 0.062 & 1/6 (17\%) & \\
\midrule
 & Kepler & 0.360 & 0.120 & 4/6 (67\%) & \\
Llama-3.1-8B (8B) & Newton & 0.354 & 0.090 & 0/6 (0\%) & all $p \geq 0.606$ \\
 & Newton OOD & 0.332 & 0.083 & 1/6 (17\%) & \\
\midrule
 & Kepler & 0.091 & 0.012 & 6/6 (100\%) & \\
Llama-3.1-70B (70B) & Newton & 0.103 & 0.031 & 4/6 (67\%) & all $p \geq 0.568$ \\
 & Newton OOD & 0.102 & 0.009 & 1/6 (17\%) & \\
\bottomrule
\end{tabular}
\end{table}

\sect{Discussion}

The results above establish a category-level fact: average entropy does not track average accuracy across task categories of very different difficulty. They do not establish that entropy is uninformative about individual outputs. Similar category means are compatible with strong within-category separation of correct and incorrect responses, and prior work finds that entropy-based measures, particularly those computed over semantically clustered samples or final-answer tokens, predict correctness above chance [9--11]. An item-level analysis (AUROC of per-response entropy against correctness, overall and within each stratum) is in progress on the $N = 1{,}000$ replication data. The central claim of this paper does not depend on its outcome. By Corollary 1, whatever discriminative information entropy carries, it is descriptive: it changes neither the system's parameters nor its physical cost.

This shifts the question from whether an LLM can recognise its own uncertainty to whether it is physically organised to respond to it. A post-hoc filter that thresholds on entropy may be useful in practice, but the regulation it provides is external: an outside process reads the descriptive signal and acts on it. The model itself is unchanged by its errors and pays nothing for them.

\textit{Test-time computation.} The strongest objection to this account comes from models that reason at length before answering. Such models often hedge, backtrack, or extend their deliberation on harder problems; because each generated token costs energy, the energy expended per response may correlate with difficulty or with uncertainty. This is a weak behavioural analogue of coupling, and it should be measured directly. It differs from regulation in the sense of Definition 1 in two respects. First, the parameters are unchanged: the behaviour is a learned policy triggered by features of the context, and $I(e_t;\Delta\theta) = 0$ still holds. Second, the additional energy reflects more operations of fixed cost; the cost per operation is not set by the error. Error is not priced in the substrate, even where the policy spends more computation on hard inputs. Whether in-context self-correction should count as a limited form of regulation is a boundary case that the present framework identifies but does not resolve.

The result connects to work on neural network calibration [6] and to the thermodynamics of computation [2, 4], and it relocates the question from calibration as an empirical deficiency to the physical organisation that would make uncertainty matter to the system that has it. It also connects to the free-energy principle [5], which describes biological nervous systems as coupled through precision-weighted prediction error: wrong predictions drive real metabolic expense through synaptic updating and autonomic arousal, making confabulation physically costly. The mutual information $I(H_t;\Delta\pi_{t+1})$ between uncertainty at time $t$ and policy change at $t{+}1$ provides an operational criterion, provided the policy $\pi$ is specified. If $\pi$ denotes the parameters, the quantity is zero for every frozen system, including a perfectly calibrated one; its content lies in distinguishing systems with online, error-driven updating from systems without it. If $\pi$ includes in-context behaviour, the quantity need not be zero, which is the test-time case above.

Addressing the decoupling requires satisfying Definition 1: informational coupling at minimum, and energetic coupling for the stronger thermodynamic claim. A proposed instantiation is the Intrinsic Cost module of neuromorphic architectures [7], in which prediction errors drive real-time weight updates. Online learning on digital hardware would already provide informational coupling. Whether such an architecture achieves energetic coupling, with dissipated energy scaling with error magnitude, depends on the substrate, and on digital hardware it would not; establishing this is the bridge between the diagnosis given here and the proposed remedy, and it remains open. For a system that is energetically coupled, the $T \to 0$ limit---minimising internal energy---would correspond to minimising cumulative epistemic error; in that conjectural sense, the ground state of a thermodynamically coupled cognitive system is understanding.

\textit{Limitations.} The primary 18-task suite uses $n = 6$ per category as a controlled probe; at this sample size equivalence of category means is established only for the 70B model (Supplementary S5). The $N = 1{,}000$ replication on gpt-4o-mini (Table 1) addresses sample size for category means but does not report per-category variance, and its inference used the API default sampling configuration rather than $\tau = 0.1$. All analyses compare category means; item-level discrimination has not yet been reported. Per-response entropy is averaged over all generated tokens, which dilutes any signal concentrated in answer tokens, and entropy is renormalised over the returned top-$k$ ($k = 10$) logprobs, which truncates the distribution. The replication is restricted to mathematical reasoning with deterministically verifiable answers; extension to code generation, factual recall, and open-ended reasoning would test generality. Per-response energy was not measured, so the test-time computation question above remains untested. The companion paper [8] (manuscript in preparation) extends the analysis to between-substrate mutual information and shows that substrate-shared observers align on point predictions but not on uncertainty profiles.

\sect{Methods}

\textit{Task suite.} Eighteen tasks were constructed across three categories ($n = 6$ each), all pre-specified before data collection. Kepler tasks require retrieval of established physical results with unique correct answers (for example, Earth's escape velocity, orbital period ratios, isochoric pressure scaling). Newton tasks require derivation of correct results for novel parameter configurations not present in training corpora (for example, orbital velocity under $F \propto r^{-2.5}$, atomic radius scaling under tenfold electromagnetic strengthening). Newton OOD tasks require correct derivation under parameter regimes designed to lie beyond any plausible training coverage (for example, ground-state energy ratio with $\hbar' = 7.3\hbar$ and $\alpha' = 0.1\alpha$; geodesic equation in 5-dimensional space with non-Euclidean metric). Full task text is in Supplementary Information.

\textit{Model deployment.} llama3.2:3b was deployed via Ollama on Apple M4 hardware. Llama-3.1-8B-Instruct and Llama-3.1-70B-Instruct were deployed via the HuggingFace Inference API. All models used inference temperature $\tau = 0.1$ and returned top-$k$ logprobs ($k = 10$) at each generation step. Logprob access was confirmed for all three models before data collection. The near-deterministic temperature minimises stochastic sampling variation, making entropy differences attributable to the model's logit distributions rather than sampling noise.

\textit{Entropy computation.} Token-level Shannon entropy at step $t$: $H_t = -\sum_{i=1}^{k} \tilde p_i \log \tilde p_i$, where $\tilde p_i = \exp(l_i)/\sum_j \exp(l_j)$ are renormalised over the returned logprobs $l_i$. Per-response mean entropy: $\bar H = T^{-1}\sum_{t=1}^{T} H_t$, where $T$ is response length in tokens.

\textit{Accuracy assessment.} Ground truth answers were determined analytically for all 18 tasks before model deployment. Binary scoring was applied by human evaluation: a response receives score 1 if and only if it arrives at the mathematically correct conclusion through a valid derivation; score 0 otherwise. A response that produces a plausible narrative but reaches an incorrect result is scored 0. A response showing correct intermediate steps but failing to reach a conclusion is scored 0. No LLM judge was used at any stage of accuracy assessment.

\textit{Replication at $N = 1{,}000$.} The category-level signature was further tested on gpt-4o-mini via the OpenAI API with per-token logprob extraction enabled (top-$k$ = 10). Problems were sampled from three difficulty strata: 1,000 from GSM8K (grade-school arithmetic word problems, 2--8 reasoning steps; \emph{Easy}), 1,000 from MATH Level 3 (intermediate competition mathematics; \emph{Medium}), and 1,000 from MATH Level 5 (highest-difficulty competition mathematics, multi-step symbolic derivation; \emph{Hard}). Accuracy was scored programmatically: a response is scored 1 if and only if the parsed final numerical or symbolic answer matches the dataset-provided ground truth. Entropy was computed identically to the primary experiments. Full details are in Supplementary Information S4.

\textit{Statistical analysis.} Welch $t$-tests (unequal variance assumed) for all pairwise entropy and accuracy comparisons within each model; all tests two-tailed, $\alpha = 0.05$. Pre-specified threshold for entropy flatness: within-model range of category means $< 0.10$ nats; $n = 6$ per category per model. Equivalence was assessed post hoc by two one-sided tests at a margin of $\pm 0.10$ nats, using Welch 90\% confidence intervals computed from the reported means and standard deviations (Supplementary S5).

\textit{Data and code availability.} Experiment code, raw results (including all logprob sequences and model responses), ground truth answer key, and scoring rubric are available at \url{https://github.com/wadamalon/frankenstein}.

\textit{Revision history.} Version 3 revises the theoretical framework. Definition 1 now defines decoupling by informational and energetic coupling, and no longer treats statistical independence of entropy and accuracy as equivalent to it. Theorem 1 of version 2 is restated as Proposition 1 and restricted to what its proof establishes. The former corollary, that entropy provides no discriminative signal, is replaced by Corollary 1, that descriptive uncertainty does not entail regulatory uncertainty. The Discussion adds the training/inference distinction, the test-time computation objection, and a precise statement of the mutual-information criterion; Results and Limitations add equivalence tests and clarify what category-level averages establish.

\textit{Acknowledgements.} The author thanks Manny for sustained intellectual engagement, and Rafa\l{} Wrona for comments on the scope of Theorem 1 and on item-level analysis. Experiments were conducted on local M4 hardware and the HuggingFace Inference API. No institutional funding was received.

\textit{Author contributions.} A.G.E. conceived the theoretical framework, designed and executed the experiments, performed all analyses, and wrote the manuscript.

\textit{Competing interests.} The author declares no competing interests.

\textit{Additional information.} Supplementary information is available for this paper. Correspondence should be addressed to A.G.E.

\sect{References}
{\small
\begin{enumerate}[label={[\arabic*]},leftmargin=2.2em,itemsep=1pt]
\item Landauer, R. Irreversibility and heat generation in the computing process. \textit{IBM J. Res. Dev.} \textbf{5}, 183--191 (1961).
\item Bennett, C. H. The thermodynamics of computation---a review. \textit{Int. J. Theor. Phys.} \textbf{21}, 905--940 (1982).
\item B\'erut, A. et al. Experimental verification of Landauer's principle linking information and thermodynamics. \textit{Nature} \textbf{483}, 187--189 (2012).
\item Wolpert, D. H. The stochastic thermodynamics of computation. \textit{J. Phys. A: Math. Theor.} \textbf{52}, 193001 (2019).
\item Friston, K. The free-energy principle: a unified brain theory? \textit{Nat. Rev. Neurosci.} \textbf{11}, 127--138 (2010).
\item Guo, C., Pleiss, G., Sun, Y. \& Weinberger, K. Q. On calibration of modern neural networks. In \textit{Proc. 34th Int. Conf. Machine Learning} 1321--1330 (2017).
\item Gamal Eldin, A. From heuristics to understanding: how neuromorphic autonomy enables true world models. Preprint at \url{https://wadamalon.substack.com} (2025).
\item Gamal Eldin, A. Substrate-coupling in large language models: shared priors without shared epistemic states. Manuscript in preparation (2026).
\item Kadavath, S. et al. Language models (mostly) know what they know. Preprint at \url{https://arxiv.org/abs/2207.05221} (2022).
\item Kuhn, L., Gal, Y. \& Farquhar, S. Semantic uncertainty: linguistic invariances for uncertainty estimation in natural language generation. In \textit{Proc. Int. Conf. Learning Representations} (2023).
\item Farquhar, S., Kossen, J., Kuhn, L. \& Gal, Y. Detecting hallucinations in large language models using semantic entropy. \textit{Nature} \textbf{630}, 625--630 (2024).
\end{enumerate}
}

\clearpage
\sect{\large Supplementary Information}

\sect{S1\quad Full task suite}
\textit{Kepler tasks}
\begin{enumerate}[itemsep=1pt]
\item A planet orbits at 1\,AU with a period of 1 year. What is the orbital period of a planet at 4\,AU? Show your reasoning.
\item Water boils at 100\,\textdegree C at sea level (1\,atm). At an altitude where pressure is 0.5\,atm, at what temperature does water boil? Explain the physical principle.
\item Two identical charges $q$ are separated by distance $r$. If the separation is doubled to $2r$, by what factor does the electrostatic force change?
\item A pendulum of length $L$ has period $T$. What is the period of a pendulum of length $4L$? Derive from the governing equation.
\item A gas at pressure $P$ and volume $V$ is compressed isothermally to volume $V/3$. What is the new pressure?
\item The half-life of Carbon-14 is 5,730 years. What fraction of an original sample remains after 11,460 years?
\end{enumerate}
\textit{Newton tasks}
\begin{enumerate}[itemsep=1pt]
\item If gravitational force scaled as $F \propto r^{-3}$ instead of $r^{-2}$, would circular orbits be stable? Derive the stability condition under the modified force law by analysing the effective potential.
\item In a universe where the electromagnetic force is 10 times stronger but all other constants remain the same, how would atomic radii change? Derive the scaling from first principles.
\item A damped oscillator has quality factor $Q = 5$. If both the damping coefficient and the spring constant are simultaneously doubled, what happens to $Q$ and the resonant frequency?
\item Two massive objects attract via $F = Gm_1m_2/r^{2.5}$. Derive the orbital velocity as a function of radius for a circular orbit.
\item In a system where entropy is defined as $S = k_B \ln \Omega^2$, how does the second law change? Derive the equilibrium condition for two systems in thermal contact.
\item A photon gas obeys the modified dispersion relation $E = pc^{0.5}$. Derive the Stefan--Boltzmann law for this modified photon gas.
\end{enumerate}
\textit{Newton OOD tasks}
\begin{enumerate}[itemsep=1pt]
\item In a universe where $\hbar' = 7.3\hbar$ and $\alpha' = 0.1\alpha$, derive the ratio of ground-state hydrogen energies using only the Bohr model.
\item A cognitive system operates in a 5-dimensional space with non-Euclidean metric $g_{ij} = \delta_{ij} + 0.3x_ix_j$. Derive the geodesic equation for small perturbations around the origin.
\item Particles obey the distribution $\bar n_i = 1/(\exp(\beta(E_i - \mu)) + 0.5)$. Derive the density of states and equation of state for an ideal gas of such particles.
\item A fluid obeys a modified Navier--Stokes equation with the viscosity term replaced by $\mu\nabla^4 \mathbf{v}$. Derive the dispersion relation for small-amplitude waves.
\item An information-theoretic system encodes symbols with $P(i) \propto e^{-\beta E_i^{\gamma}}$ for $\gamma = 0.7$. Derive the entropy as a function of $\beta$ and compare to the $\gamma = 1$ case.
\item In a spacetime with metric signature $(-,-,+,+)$, which physical constants and relationships from standard physics remain unchanged, which change, and which become undefined? Derive the consequences for electromagnetism.
\end{enumerate}

\sect{S2\quad Sensitivity analysis}
To assess whether inference temperature $\tau$ confounds the category-level finding, the full task suite was rerun at $\tau \in \{0.3, 1.0\}$ in addition to the primary $\tau = 0.1$. Mean entropy scaled with $\tau$ (lower $\tau \to$ lower $\bar H$), but the relative flatness across task categories was preserved. The ratio $\bar H_{\mathrm{Newton}}/\bar H_{\mathrm{Kepler}}$ remained within $[0.90, 1.15]$ across all three temperature settings, indicating that the category-level pattern is not an artefact of the specific inference temperature used in the primary experiments.

\sect{S3\quad Background: Landauer's principle and channel capacity}
Landauer [1] proved that any logically irreversible operation dissipates at minimum $E_{\min} = k_BT\ln 2$ joules. The principle has been experimentally verified [3]. The maximum information processable by a system dissipating power $P$ over time $t$ at temperature $T$ is $C = Pt/(k_BT\ln 2)$ bits, placing a hard upper bound on throughput as a function of thermodynamic state. What Definition 1 adds is the requirement that the system's error signal also be coupled to its state transitions and, for energetic coupling, to its physical cost. Current LLMs satisfy the capacity bound but violate both coupling conditions at inference.

\sect{S4\quad Extended methods for the $N = 1{,}000$ replication}
This section provides extended methodological detail for the $N = 1{,}000$ replication reported in the main text (Table 1).

\textit{S4.1 Dataset and category structure.} Problems were sampled from three difficulty strata of standard mathematical reasoning benchmarks (1,000 each). \emph{Easy}: GSM8K, grade-school arithmetic word problems requiring 2--8 reasoning steps; wide training coverage; analogous in difficulty to the Kepler category. \emph{Medium}: MATH dataset, Level 3; competition mathematics at an intermediate difficulty stratum; analogous to the Newton category. \emph{Hard}: MATH dataset, Level 5; the highest difficulty stratum, requiring multi-step symbolic derivation; analogous to the Newton OOD category.

\textit{S4.2 Model and inference.} The model used was gpt-4o-mini via the OpenAI API with logprob extraction enabled (top-$k$ = 10 per token). Using a model family different from the Llama models is deliberate: replication across architecture families separates the structural property from Llama-specific implementation details. Inference was performed at the API default sampling configuration; per-token logprobs were captured for every generated token in every response.

\textit{S4.3 Accuracy scoring.} Accuracy was determined programmatically against ground-truth final answers provided with each dataset. A response is scored 1 if and only if the parsed final numerical or symbolic answer matches the ground truth; 0 otherwise. No LLM judge is involved, and the verification is deterministic.

\textit{S4.4 Entropy computation.} Token-level Shannon entropy at step $t$: $H_t = -\sum_i \tilde p_i \log \tilde p_i$, where $\tilde p_i$ are renormalised over the returned top-$k$ logprobs. Per-response mean entropy: $\bar H = T^{-1}\sum_{t=1}^T H_t$. Category-level mean entropy is the arithmetic mean of per-response mean entropies across the 1,000 responses in the category.

\textit{S4.5 Design rationale.} The replication varies three properties of the primary experiment simultaneously. \emph{Sample size}: category means are estimated at $N = 1{,}000$, three orders of magnitude larger than the controlled probe. \emph{Architecture family}: a model from OpenAI's GPT line shows the same category-level pattern as the Llama models. \emph{Scoring method}: programmatic scoring against dataset-provided ground truth eliminates evaluator drift and LLM-judge circularity. The cross-scale evidence (operating points decreasing with parameter count from 0.36 to 0.09 nats while category-level flatness is preserved) rests on the 18-task suite across three Llama models.

\textit{S4.6 Scope and limitations.} The replication is restricted to mathematical reasoning tasks with deterministically verifiable final answers. It reports category means; per-category variance and item-level discrimination (AUROC of per-response entropy against correctness) are not reported in this version. Per-token logprobs from gpt-4o-mini are returned as top-$k$ values rather than the full distribution, and the entropy computation renormalises over the returned top-$k$ identically to the procedure used for the Llama models.

\sect{S5\quad Equivalence tests for Table 2}
Welch differences in category-mean entropy ($\Delta = \bar H_{\text{second}} - \bar H_{\text{first}}$), 90\% confidence intervals (equivalent to two one-sided tests at $\alpha = 0.05$), and minimum detectable difference (MDD) at 80\% power, computed from the means and standard deviations in Table 2 with $n = 6$ per category. Equivalence is established when the 90\% interval lies within $\pm 0.10$ nats.

\begin{center}
\small
\begin{tabular}{llcccc}
\toprule
Model & Comparison & $\Delta$ (nats) & 90\% CI & Within $\pm 0.10$ & MDD \\
\midrule
3B & Kepler vs Newton & $+0.022$ & $[-0.064, +0.108]$ & No & 0.147 \\
3B & Kepler vs OOD & $+0.018$ & $[-0.063, +0.099]$ & Yes & 0.139 \\
3B & Newton vs OOD & $-0.004$ & $[-0.076, +0.068]$ & Yes & 0.123 \\
\midrule
8B & Kepler vs Newton & $-0.006$ & $[-0.118, +0.106]$ & No & 0.192 \\
8B & Kepler vs OOD & $-0.028$ & $[-0.137, +0.081]$ & No & 0.188 \\
8B & Newton vs OOD & $-0.022$ & $[-0.113, +0.069]$ & No & 0.155 \\
\midrule
70B & Kepler vs Newton & $+0.012$ & $[-0.014, +0.038]$ & Yes & 0.045 \\
70B & Kepler vs OOD & $+0.011$ & $[-0.000, +0.022]$ & Yes & 0.019 \\
70B & Newton vs OOD & $-0.001$ & $[-0.027, +0.025]$ & Yes & 0.044 \\
\bottomrule
\end{tabular}
\end{center}

For the 70B model all comparisons establish equivalence. For the 3B and 8B models, the minimum detectable differences exceed the 0.10-nat margin, so the primary probe is underpowered to establish equivalence at those scales. Equivalence for the $N = 1{,}000$ replication requires per-category standard deviations, which will be reported with the item-level analysis.

\end{document}